\documentclass[runningheads,a4paper]{llncs}

\usepackage[utf8]{inputenc}
\usepackage[T1]{fontenc}    
\usepackage{hyperref} 
\usepackage{graphicx,times,amssymb,float,xcolor}      
\usepackage{url}   
\usepackage{epstopdf}         
\usepackage{booktabs}       
\usepackage{amsfonts}      
\usepackage{microtype}      
\usepackage{bm}
\usepackage{caption}
\DeclareCaptionType{noticebox}
\usepackage{subcaption}
\usepackage{pdfpages}
\usepackage{cleveref}
\usepackage{amsmath}

\newcommand\norm[1]{\left\lVert#1\right\rVert}
\newcommand\abs[1]{\left\lvert#1\right\rvert}
\newcommand{\ie}{i.e.}

\DeclareMathOperator*{\argmin}{arg\,min}

\usepackage{url}
\urldef{\mailsa}\path|{alfred.hofmann, ursula.barth, ingrid.haas, frank.holzwarth,|
\urldef{\mailsb}\path|anna.kramer, leonie.kunz, christine.reiss, nicole.sator,|
\urldef{\mailsc}\path|erika.siebert-cole, peter.strasser, lncs}@springer.com|    
\newcommand{\keywords}[1]{\par\addvspace\baselineskip
\noindent\keywordname\enspace\ignorespaces#1}

\newcommand{\ah}[1]{{\color{red}#1}}

\begin{document}
\allowdisplaybreaks
\mainmatter  

\title{Theoretical Analysis of Domain Adaptation with Optimal Transport}
\titlerunning{Theoretical Analysis of Domain Adaptation with Optimal Transport}

%
%

\author{Ievgen Redko\inst{1} \and Amaury Habrard\inst{2} \and Marc Sebban \inst{2}}

%
%
%

\institute{Univ.Lyon, INSA‐Lyon, Universit\'e Claude Bernard Lyon 1, UJM-Saint Etienne\\
CNRS, Inserm, CREATIS UMR 5220, U1206\\
F-69266, LYON, France\\
\email{ievgen.redko@creatis.insa-lyon.fr}\\ 
\and
Univ. Lyon, UJM-Saint-Etienne\\
   CNRS, Lab. Hubert Curien UMR 5516\\
   F-42023, SAINT-ETIENNE, France\\
   \email{\{amaury.habrard,marc.sebban\}@univ.st-etienne.fr}}


%

%
%

\maketitle

\begin{abstract}
Domain adaptation (DA) is an important and emerging field of machine learning that tackles the problem occurring when the distributions of training (source domain) and test (target domain) data are similar but different. This kind of learning paradigm is of vital importance for future advances as it allows a learner to generalize the knowledge across different tasks. Current theoretical results show that the efficiency of DA algorithms depends on their capacity of minimizing the divergence between source and target probability distributions. In this paper, we provide a theoretical study on the advantages that concepts borrowed from optimal transportation theory \cite{monge_81} can bring to DA. In particular, we show that the Wasserstein metric can be used as a divergence measure between distributions to obtain generalization guarantees for three different learning settings: (i) classic DA with unsupervised target data (ii) DA combining source and target labeled data, (iii) multiple source DA. Based on the obtained results, we motivate the use of the regularized optimal transport and provide some algorithmic insights for multi-source domain adaptation. We also show when this theoretical analysis can lead to tighter inequalities than those of other existing frameworks. We believe that these results open the door to novel ideas and directions for DA.  
\end{abstract}

\section{Introduction}
Many results in statistical learning theory study the problem of estimating the probability that a hypothesis chosen from a given hypothesis class can achieve a small true risk. This probability is often expressed in the form of generalization bounds on the true risk obtained using concentration inequalities with respect to (w.r.t.) some hypothesis class. Classic generalization bounds make the assumption that training and test data follow the same distribution. This assumption, however, can be violated in many real-world applications (e.g., in computer vision, language processing or speech recognition) where training and test data actually follow a related but different probability distribution. One may think of an example, where a spam filter is learned based on the abundant annotated data collected for one user and is further applied for newly registered user with different preferences. In this case, the performance of the spam filter will deteriorate as it does not take into account the mismatch between the underlying probability distributions. The need for algorithms tackling this problem has led to the emergence of a new field in machine learning called domain adaptation (DA), subfield of transfer learning \cite{Pan:2010}, where the source (training) and target (test) distributions are not assumed to be the same.
From a theoretical point of view, existing generalization guarantees for DA are expressed in the form of bounds over the target risk involving the source risk, a divergence between domains and a term $\lambda$ evaluating the capability of the considered hypothesis class to solve the problem, often expressed as a joint error of the ideal hypothesis between the two domains. In this context, minimizing the divergence between distributions is a key factor for the potential success of DA algorithms. Among the most striking results, existing generalization bounds based on the H-divergence \cite{Ben-david07analysisof} or the discrepancy distance \cite{DBLP:conf/colt/MansourMR09} have also an interesting property of being able to link the divergence between the probability distributions of two domains w.r.t. the considered class of hypothesis.

Despite their advantages, the above mentioned divergences do not directly take into account the geometry of the data distribution. Recently, \cite{courty14a,courty16a} has proposed to tackle this drawback by solving the DA problem using ideas from optimal transportation (OT) theory. Their paper proposes an algorithm that aims to reduce the divergence between two domains by minimizing the Wasserstein distance between their distributions. This idea has a very appealing and intuitive interpretation based on the transport of one domain to another. The transportation plan solving OT problem takes into account the geometry of the data by means of an associated cost function which is based on the Euclidean distance between examples. Furthermore, it is naturally defined as an infimum problem over all feasible solutions. An interesting property of this approach is that the resulting solution given by a joint probability distribution allows one to obtain the new projection of the instances of one domain into another directly without being restricted to a particular hypothesis class. This independence from the hypothesis class means that this solution not only ensures successful adaptation but also influences the capability term $\lambda$. While showing very promising experimental results, it turns out that this approach, however, has no theoretical guarantees. This paper aims to bridge this gap by presenting contributions covering three DA settings: (i) classic unsupervised DA where the learner has only access to labeled source data and unsupervised target instances, (ii) DA where one has access to labeled data from both source and target domains, (iii) multi-source DA where labeled instances for a set of distinct source domains (more than 2) are available. 
We provide new theoretical guarantees in the form of generalization bounds for these three settings based on the Wasserstein distance thus justifying its use in DA. According to \cite{opac-b1129524}, the Wasserstein distance is rather strong and can be combined with smoothness bounds to obtain convergences in other distances. This important advantage of Wasserstein distance leads to tighter bounds in comparison to other state-of-the-art results and is more computationally attractive.

The rest of this paper is organized as follows: Section 2 is devoted to the presentation of optimal transport and its application in DA. In Section 3, we present the generalization bounds for DA with the Wasserstein distance for both single- and multi-source learning scenarios. Finally, we conclude our paper in Section 4.

\section{Definitions and notations}
In this section, we first present the formalization of the Monge-Kantorovich \cite{kantorovich} optimization problem and show how optimal transportation problem found its application in DA. 

\subsection{Optimal transport} 
Optimal transportation theory was first introduced in \cite{monge_81} to study the problem of resource allocation. Assuming that we have a set of factories and a set of mines, the goal of optimal transportation is to move the ore from mines to factories in an optimal way, \ie , by minimizing the overall transport cost. More formally, let $\Omega  \subseteq \mathbb{R}^d$ be a measurable space and denote by $\mathcal{P}\left(\Omega\right)$ the set of all probability measures over $\Omega$. Given two probability measures $\mu_S, \mu_T \in  \mathcal{P}\left(\Omega \right)$, the Monge-Kantorovich problem consists in finding a probabilistic coupling $\gamma$ defined as a joint probability measure over $\Omega \times \Omega$ with marginals $\mu_S$ and $\mu_T$ for all $x,y \in \Omega$ that minimizes the cost of transport w.r.t. some function $c: \Omega \times \Omega \rightarrow \mathbb{R}_+$:
\begin{align*}
&\underset{\gamma}{\arg \min} \int_{\Omega_1 \times \Omega_2} c(\boldsymbol{x},\boldsymbol{y})^pd\gamma(\boldsymbol{x},\boldsymbol{y})\\
&\text{s.t.} \ \boldsymbol{P}^{\Omega_1}\# \gamma = \mu_S, \boldsymbol{P}^{\Omega_2}\# \gamma = \mu_T,
\end{align*}
where $\boldsymbol{P}^{\Omega_i}$ is the projection over $\Omega_i$ and $\#$ denotes the pushforward measure. 
This problem admits a unique solution $\gamma_0$ which allows us to define the Wasserstein distance of order $p$ between $\mu_S$ and $\mu_T$ for any $p \in [1; +\infty]$ as follows:
$$W_p^p(\mu_S,\mu_T) = \inf_{\gamma \in \Pi(\mu_S, \mu_T)} \int_{\Omega \times \Omega} c(\boldsymbol{x},\boldsymbol{y})^pd\gamma(\boldsymbol{x},\boldsymbol{y}),$$
where $c: \Omega \times \Omega \rightarrow \mathbb{R}^+$ is a cost function for transporting one unit of mass $\bm{x}$ to $\bm{y}$ and $\Pi(\mu_S, \mu_T)$ is a collection of all joint probability measures on $\Omega \times \Omega$ with marginals $\mu_S$ and $\mu_T$.
\begin{remark} 
In what follows, we consider only the case $p=1$ but all the obtained results can be easily extended to the case $p>1$ using H\"older inequality implying for every $p\leq q \Rightarrow W_p \leq W_q$.   
\end{remark}
In the discrete case, when one deals with empirical measures $\hat{\mu}_S = \frac{1}{N_S}\sum_{i=1}^{N_S}\delta_{x_S^i}$ and $\hat{\mu}_T = \frac{1}{N_T}\sum_{i=1}^{N_T}\delta_{x_T^i}$ represented by the uniformly weighted sums of $N_S$ and $N_T$ Diracs with mass at locations $x_S^i$ and $x_T^i$ respectively, Monge-Kantorovich problem is defined in terms of the inner product between the coupling matrix $\gamma$ and the cost matrix $C$:
$$W_1(\hat{\mu}_S, \hat{\mu}_T) = \min_{\gamma \in \Pi(\hat{\mu}_S, \hat{\mu}_T)}\langle C, \gamma\rangle_F$$
where $\langle \cdot \text{,} \cdot \rangle_F$ is the Frobenius dot product, $\Pi(\hat{\mu}_S, \hat{\mu}_T) = \lbrace \gamma \in \mathbb{R}^{N_S \times N_T}_+ \vert \gamma \bm{1} = \hat{\mu}_S, \gamma^T \bm{1} = \hat{\mu}_T\rbrace$ is a set of doubly stochastic matrices and $C$ is a dissimilarity matrix, \ie, $C_{ij} = c(x_S^i,x_T^j)$, defining the energy needed to move a probability mass from $x_S^i$ to $x_T^j$. 
Figure 1 shows how the solution of optimal transport between two point-clouds can look like.

\begin{figure}
\begin{center}
  \includegraphics[scale=0.35]{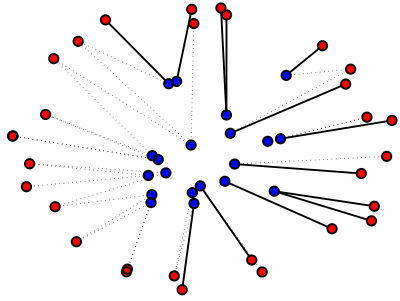} 
   \caption{Blue points are generated to lie inside a square with a side length equal to 1. Red points are generated inside an annulus containing the square. Solution of the regularized optimal transport problem is visualized by plotting dashed and solid lines that correspond to the large and small values given by the optimal coupling matrix $\gamma$.} 
  \label{fig:sfig2}
\end{center}
\end{figure}

It turns out that the Wasserstein distance has been successfully used in various applications, for instance: computer vision \cite{Rubner:2000:EMD:365875.365881}, texture analysis \cite{conf/scalespace/RabinPDB11}, tomographic reconstruction \cite{abraham_2015} and clustering \cite{DBLP:conf/icml/CuturiD14}. The huge success of algorithms based on this distance is due to \cite{conf/nips/Cuturi13} who introduced an entropy-regularized version of optimal transport that can be optimized efficiently using matrix scaling algorithm. We are now ready to present the application of OT to DA below. 

\subsection{Domain adaptation and optimal transport}
The problem of DA is formalized as follows: we define a domain as a pair consisting of a distribution $\mu_D$ on $\Omega$ and a labeling function $f_D: \Omega \rightarrow [0,1]$. A hypothesis class $H$ is a set of functions so that $\forall h \in H, h : \Omega \rightarrow \lbrace0,1\rbrace$. 
\begin{definition}
Given a convex loss-function $l$, the probability according to the distribution $\mu_D$ that a hypothesis $h \in H$ disagrees with a labeling function $f_D$ (which can also be a hypothesis) is defined as
$$\epsilon_D (h,f_D) = \mathbb{E}_{x \sim \mu_D} \left[l(h(x),f_D(x))\right].$$
\end{definition}
When the source and target error functions are defined w.r.t. $h$ and $f_S$ or $f_T$, we use the shorthand $\epsilon_S (h, f_S) = \epsilon_S (h)$ and $\epsilon_T (h, f_T) = \epsilon_T (h)$. We further denote by $\langle \mu_S, f_S \rangle$ the source domain and $\langle \mu_T, f_T \rangle$ the target domain. The ultimate goal of DA then is to learn a good hypothesis $h$ in $\langle \mu_S, f_S \rangle$ that has a good performance in $\langle \mu_T, f_T \rangle$.

In unsupervised DA problem, one usually has access to a set of source data instances $\bm{X_S} = \{\bm{x}^{i}_S \in \mathbb{R}^d \}_{i=1}^{N_S}$ associated with labels $\{ y^i_S\}_{i=1}^{N_S}$ and a set of unlabeled target data instances $\bm{X_T} = \{\bm{x}^{i}_T \in \mathbb{R}^d \}_{i=1}^{N_T}$. Contrary to the classic learning paradigm, unsupervised DA assumes that the marginal distributions of $\bm{X_S}$ and $\bm{X_T}$ are different and given by $\mu_S, \mu_T \in  \mathcal{P}\left(\Omega\right)$. 

For the first time, optimal transportation problem was applied to DA in \cite{courty14a,courty16a}. The main underlying idea of their work is to find a coupling matrix that efficiently transports source samples to target ones by solving the following optimization problem :
$$\gamma_o = \argmin_{\gamma \in \Pi(\hat{\mu}_S,\hat{\mu}_T)}\langle C, \gamma\rangle_F.$$
Once the optimal coupling $\gamma_o$ is found, source samples $\bm{X_S}$ can be transformed into target aligned source samples $\bm{\hat{X}_S}$ using the following equation 
$$\bm{\hat{X}_S} = \text{diag}((\gamma_o \bm{1})^{-1})\gamma_o \bm{X_T}.$$
The use of Wasserstein distance here has an important advantage over other distances used in DA (see Section 3.4) as it preserves the topology of the data and admits a rather efficient estimation as mentioned above. Furthermore, as shown in \cite{courty14a,courty16a}, it improves current state-of-the-art results on benchmark computer vision data sets and has a very appealing intuition behind.

\section{Generalization bounds with Wasserstein distance}
In this section, we introduce generalization bounds for the target error when the divergence between tasks' distributions is measured by the Wasserstein distance. 
\subsection{A bound relating the source and target error}
We first consider the case of unsupervised DA where no labelled data are available in the target domain. We start with a lemma that relates the Wasserstein metric with the source and target error functions for an arbitrary pair of hypothesis. Then, we show how the target error can be bounded by the Wasserstein distance for empirical measures. We first present the Lemma that introduces Wasserstein distance to relate the source and target error functions in a Reproducing Kernel Hilbert Space. 
\begin{lemma}
Let $\mu_S, \mu_T \in  \mathcal{P}\left(\Omega\right)$ be two probability measures on $\mathbb{R}^d$. Assume that the cost function $c(\bm{x},\bm{y}) = \Vert \phi(\bm{x}) - \phi(\bm{y}) \Vert_{\mathcal{H}_{k_l}}$, where $\mathcal{H}_{k_l}$ is a Reproducing Kernel Hilbert Space (RKHS) equipped with kernel $k_l: \Omega \times \Omega \rightarrow \mathbb{R}$ induced by $\phi: \Omega \rightarrow \mathcal{H}_{k_l}$ and $k_l(\bm{x}, \bm{y}) = \langle \phi(\bm{x}), \phi(\bm{y}) \rangle_{\mathcal{H}_{k_l}}$. Assume further that the loss function $l_{h,f}:x \longrightarrow l(h(x),f(x))$ is convex, symmetric, bounded, obeys the triangular equality and has the parametric form $\vert h(x) - f(x) \vert^q$ for some $q > 0$. Assume also that kernel $k_l$ in the RKHS $\mathcal{H}_{k_l}$ is square-root integrable w.r.t. both $\mu_S,\mu_T$ for all  $\mu_S,\mu_T \in \mathcal{P}(\Omega)$ where $\Omega$ is separable and $0\leq k_l(\bm{x},\bm{y}) \leq K, \forall \ \bm{x},\bm{y} \in \Omega$. Then the following holds
$$ \epsilon_T (h , h')\leq \epsilon_S (h , h') + W_1(\mu_S,\mu_T)$$
for every hypothesis $h', h$.
\label{trm:mmd_w}
\end{lemma}
\begin{proof}
As this Lemma plays a key role in the following sections, we give its proof here. We assume that $l_{h,f}:x \longrightarrow l(h(x),f(x))$ in the definition of $\epsilon(h)$ is a convex loss-function defined  $\forall h,f \in \mathcal{F}$ where $\mathcal{F}$ is a unit ball in the RKHS $\mathcal{H}_k$. Considering that $h,f \in \mathcal{F}$, the loss function $l$ is a non-linear mapping of the RKHS $\mathcal{H}_{k}$ for the family of losses $l(h(x),f(x)) = \vert h(x) - f(x) \vert^q$\footnote{If $h,f \in \mathcal{H}$ then $h-f \in \mathcal{H}$ implying that $l(h(x),f(x)) = \vert h(x) - f(x) \vert^q$ is a nonlinear transform for $h-f \in \mathcal{H}$.}. Using results from \cite{Saitoh}, one may show that $l_{h,f}$ also belongs to the RKHS $\mathcal{H}_{k_l}$ admitting the reproducing kernel $k_l$ and that its norm obeys the following inequality:
$$\vert \vert l_{h,f} \vert \vert_{\mathcal{H}_{k_l}}^2 \leq \vert \vert h - f \vert \vert_{\mathcal{H}_k}^{2q}.$$
This result gives us two important properties of $l_{f,h}$ that we use further:
\begin{itemize}
\item $l_{h,f}$ belongs to the RKHS that allows us to use the reproducing property;
\item the norm $\vert \vert l_{h,f} \vert \vert_{\mathcal{H}_{k_l}}$ is bounded. 
\end{itemize}
For simplicity, we can assume that $\vert \vert l_{h,f} \vert \vert_{\mathcal{H}_{k_l}}$ is bounded by 1. This assumption can be verified by imposing the appropriate bounds on the norms of $h$ and $f$ and is easily extendable to the case when $\vert \vert l_{h,f} \vert \vert_{\mathcal{H}_{k_l}} \leq M$ by scaling as explained in \cite[Proposition 2]{DBLP:conf/colt/MansourMR09}. We also note that $q$ does not necessarily have to appear in the final result as we seek to bound the norm of $l$ and not to give an explicit expression for it in terms of $\Vert h \Vert_{\mathcal{H}_{k}}, \Vert f \Vert_{\mathcal{H}_{k}}$ and $q$.
Now the error function defined above can be also expressed in terms of the inner product in the corresponding Hilbert space, i.e\footnote{For the sake of simplicity, we will further write $\mathcal{H}$ meaning $\mathcal{H}_{k_l}$ and $l$ meaning $l_{f,h}$.}:
$$\epsilon_S (h, f_S) = \mathbb{E}_{x \sim \mu_S} [l(h(x),f_S(x))] = \mathbb{E}_{x \sim \mu_S} [\langle\phi(x),l\rangle_\mathcal{H}].$$
We define the target error in the same manner: 
$$\epsilon_T (h, f_T) = \mathbb{E}_{y \sim \mu_T} [l(h(y),f_T(y))] = \mathbb{E}_{y \sim \mu_T} [\langle\phi(y),l\rangle_\mathcal{H}].$$
With the definitions introduced above, the following holds:
\begin{align*}
\epsilon_T (h , h') &   = \epsilon_T (h , h') + \epsilon_S (h , h') - \epsilon_S (h , h') \\
& = \epsilon_S (h , h') + \mathbb{E}_{y \sim \mu_T} [\langle\phi(y),l\rangle_\mathcal{H}] - \mathbb{E}_{x \sim \mu_S} [\langle \phi(x),l\rangle_\mathcal{H}] \\
& = \epsilon_S (h , h') + \langle \mathbb{E}_{y \sim \mu_T} [\phi(y)] - \mathbb{E}_{x \sim \mu_S} [ \phi(x)] ,l\rangle_\mathcal{H}\\
& \leq \epsilon_S (h , h') + \Vert l \Vert_\mathcal{H} \Vert\mathbb{E}_{y \sim \mu_T} [\phi(y)] - \mathbb{E}_{x \sim \mu_S} [ \phi(x)]\Vert_\mathcal{H}\\
& \leq \epsilon_S (h , h') + \Vert \int_{\Omega} \phi d(\mu_S - \mu_T) \Vert_\mathcal{H}.
\end{align*}
Second line is obtained by using the reproducing property applied to $l$, third line follows from the properties of the expected value. Fourth line here is due to the properties of the inner-product while fifth line is due to $\vert \vert l_{h,f} \vert \vert_{\mathcal{H}} \leq 1$.
Now using the definition of the joint distribution we have the following:
\begin{align*}
\Vert \int_{\Omega} \phi d(\mu_S - \mu_T) \Vert_\mathcal{H} & = \Vert \int_{\Omega \times \Omega} (\phi(\bm{x}) - \phi(\bm{y})) d\gamma(\bm{x},\bm{y}) \Vert_\mathcal{H}\\
& \leq \int_{\Omega \times \Omega } \Vert \phi(\bm{x}) - \phi(\bm{y}) \Vert_{\mathcal{H}} d\gamma(\bm{x},\bm{y}).
\end{align*}
As the last inequality holds for any $\gamma$, we obtain the final result by taking the infimum over $\gamma$ from the right-hand side, \ie : 
\begin{align*}
\int_{\Omega} \phi d(\mu_S - \mu_T) \Vert_\mathcal{H} \leq \inf_{\gamma \in \Pi(\mu_S, \mu_T)} \int_{\Omega \times \Omega} \Vert \phi(\bm{x}) - \phi(\bm{y}) \Vert_{\mathcal{H}} d\gamma(\bm{x},\bm{y}).
\end{align*}
which gives
$$\epsilon_T (h , h') \leq \epsilon_S (h , h') + W_1(\mu_S,\mu_T).$$
$\Box$
\end{proof}
\begin{remark}
We note that the functional form of the loss-function $l(h(x),f(x)) = \vert h(x) - f(x) \vert^q$ is just an example that was used as the basis for the proof. According to \cite[Appendix 2]{Saitoh}, we may also consider more general nonlinear transformations of $h$ and $f$ that satisfy the assumption imposed on $l_{h,f}$ above. These transformations may include a product of hypothesis and labeling functions and thus the proposed results is valid for hinge-loss too.
\end{remark}
This lemma makes use of the Wasserstein distance to relate the source and target errors. The assumption made here is to specify that the cost function $c(\bm{x},\bm{y}) = \Vert \phi(\bm{x}) - \phi(\bm{y}) \Vert_{\mathcal{H}}$. While it may seem too restrictive, this assumption is, in fact, not that strong. Using the properties of the inner-product, one has:
\begin{align*}
\Vert \phi(\bm{x}) - \phi(\bm{y}) \Vert_{\mathcal{H}} &= \sqrt{\langle \phi(\bm{x}) - \phi(\bm{y}), \phi(\bm{x}) - \phi(\bm{y}) \rangle_{\mathcal{H}}} \\
&= \sqrt{k(\bm{x},\bm{x}) -2k(\bm{x},\bm{y})+k(\bm{x},\bm{y})}.
\end{align*}

Now it can be shown that for any given positive-definite kernel $k$ there is a distance $c$ (used as a cost function in our case) that generates it and vice versa (see Lemma 12 from \cite{journals/corr/abs-1207-6076}).
 
In order to prove our next theorem, we present first an important result showing the convergence of the empirical measure $\hat{\mu}$ to its true associated measure w.r.t. the Wasserstein metric. This concentration guarantee allows us to propose generalization bounds based on the Wasserstein distance for finite samples rather than true population measures. Following \cite{Bolley:2007:QCI}, it can be specialized for the case of $W_1$ as follows\footnote{We present the original version of this Theorem in the Supplementary material.}
\begin{theorem}[\cite{Bolley:2007:QCI}, Theorem 1.1]\label{trm:bolley}
Let $\mu$ be a probability measure in $\mathbb{R}^d$ so that for some $\alpha>0$, we have that $\int_{\mathbb{R}^d} e^{\alpha\Vert x\Vert^2}d\mu<\infty$ and $\hat{\mu} = \frac{1}{N}\sum_{i=1}^{N} \delta_{x_i}$ be its associated empirical measure defined on a sample of independent variables $\{ x_i \}_{i=1}^N$ drawn from $\mu$. Then for any $d'>d$ and $\varsigma' < \sqrt{2}$ there exists some constant $N_0$ depending on $d'$ and some square exponential moment of $\mu$ such that for any $\varepsilon > 0$ and $N \geq N_0 \max(\varepsilon^{-(d'+2)},1)$
$$\mathbb{P} \left[ W_1(\mu, \hat{\mu}) > \varepsilon \right] \leq \exp\left(-\frac{\varsigma'}{2} N \varepsilon^2\right),$$
\label{trm_concen}
where $d', \varsigma'$ can be calculated explicitly.
\end{theorem}
The convergence guarantee of this theorem can be further strengthened as shown in \cite{fournier:hal-00915365} but we prefer this version for the ease of reading.  
We can now use it in combination with the previous Lemma to prove the following theorem.
\begin{theorem}
Under the assumptions of Lemma \ref{trm:mmd_w}, let $\mathbf{X_S}$ and $\mathbf{X_T}$ be two samples of size $N_S$ and $N_T$ drawn i.i.d. from $\mu_S$ and $\mu_T$ respectively. Let $\hat{\mu}_S = \frac{1}{N_S}\sum_{i=1}^{N_S} \delta_{x_S^i}$ and $\hat{\mu}_T = \frac{1}{N_T}\sum_{i=1}^{N_T} \delta_{x_T^i}$ be the associated empirical measures. Then for any $d'>d$ and $\varsigma' < \sqrt{2}$ there exists some constant $N_0$ depending on $d'$ such that for any $\delta > 0$ and $\min(N_S,N_T) \geq N_0 \max(\delta^{-(d'+2)},1)$ with probability at least $1-\delta$ for all $h$ the following holds:
\begin{align*}
\epsilon_T (h)\leq \epsilon_S (h) &+ W_1(\hat{\mu}_S, \hat{\mu}_T) + \sqrt{2\log\left(\frac{1}{\delta}\right)/\varsigma'}\left(\sqrt{\frac{1}{N_S}}+\sqrt{\frac{1}{N_T}}\right) + \lambda ,
\end{align*}
where $\lambda$ is the combined error of the ideal hypothesis $h^*$ that minimizes the combined error of $\epsilon_S(h)+\epsilon_T(h)$.
\label{trm:our2}
\end{theorem}
\begin{proof}
\begin{align*}
\epsilon_T (h) & \leq \epsilon_T (h^*) + \epsilon_T (h^*,h) = \epsilon_T (h^*) + \epsilon_S (h,h^*) +  \epsilon_T (h^*,h) - \epsilon_S (h,h^*)\\
& \leq \epsilon_T (h^*) + \epsilon_S (h,h^*) + W_1(\mu_S, \mu_T) \\
& \leq \epsilon_T (h^*) + \epsilon_S (h) + \epsilon_S (h^*) + W_1(\mu_S, \mu_T) \\
& =  \epsilon_S (h) + W_1(\mu_S, \mu_T) + \lambda \\
& \leq \epsilon_S (h) + W_1(\mu_S, \hat{\mu}_S) + W_1(\hat{\mu}_S, \mu_T) + \lambda \\
& \leq \epsilon_S (h) + \sqrt{2\log\left(\frac{1}{\delta}\right)/N_S\varsigma'} + W_1(\hat{\mu}_S, \hat{\mu}_T) + W_1(\hat{\mu}_T,\mu_T) + \lambda \\
& \leq \epsilon_S (h) + W_1(\hat{\mu}_S, \hat{\mu}_T) + \lambda+ \sqrt{2\log\left(\frac{1}{\delta}\right)/\varsigma'}\left(\sqrt{\frac{1}{N_S}}+\sqrt{\frac{1}{N_T}}\right).
\end{align*}
Second and fourth lines are obtained using the triangular inequality applied to the error function. Third inequality is a consequence of Lemma 1. Fifth line follows from the definition of $\lambda$, sixth, seventh and eighth lines use the fact that Wasserstein metric is a proper distance and Theorem 1. $\Box$
\end{proof}

A first immediate consequence of this theorem is that it justifies the use of the optimal transportation in DA context. However, we would like to clarify the fact that the bound does not suggest that minimization of the Wasserstein distance can be done independently from the minimization of the source error nor it says that the joint error given by the lambda term becomes small. First, it is clear that the result of $W_1$ minimization provides a transport of the source to the target such as $W_1$ becomes small when computing the distance between newly transported sources and target instances. Under the hypothesis that class labeling is preserved by transport, \ie $P_{\text{source}}(y|x_s)=P_{\text{target}}(y|\text{Transport}(x_s))$, the adaptation can be possible by minimizing $W_1$ only. However, this is not a reasonable assumption in practice. Indeed, by minimizing the $W_1$ distance only, it is possible that the obtained transformation transports one positive and one negative source instance to the same target point and then the empirical source error cannot be properly minimized. Additionally, the joint error will be affected since no classifier will be able to separate these source points. We can also think of an extreme case where the positive source examples are transported to negative target instances, in that case the joint error $\lambda$ will be dramatically affected.
A solution is then to regularize the transport to help the minimization of the source error which can be seen as a kind of joint optimization. This idea was partially implemented as a class-labeled regularization term added to the original optimal transport formulation in \cite{courty14a,courty16a} and showed good empirical results in practice. The proposed regularized optimization problem reads
$$\min_{\gamma \in \Pi(\hat{\mu}_S,\hat{\mu}_T)}\langle C, \gamma\rangle_F - \frac{1}{\lambda}E(\gamma) + \eta \sum_j \sum_\mathcal{L} \Vert \gamma(I_\mathcal{L}, j)\Vert_q^p.$$ 
Here, the second term $E(\gamma) = -\sum_{i,j}^{N_S,N_T} \gamma_{i,j}\log(\gamma_{i,j})$ is the regularization term that allows one to solve optimal transportation problem efficiently using Sinkhorn-Knopp matrix scaling algorithm \cite{sinknopp_67}. Second regularization term $\eta \sum_j \sum_c \Vert \gamma(I_c, j)\Vert_q^p$ is used to restrict source examples of different classes to be transported to the same target examples by promoting group sparsity in the matrix $\gamma$ thanks to $\Vert \cdot \Vert^p_q$ with $q = 1$ and $p = \frac{1}{2}$. In some way, this regularization term influences the capability term by ensuring the existence of a good hypothesis that will be able to be discriminant on both source and target domains data. Another recent paper of \cite{PerrotCFH16} also suggests that transport regularization is important for the use of OT in domain adaptation tasks. Thus, we conclude that the regularized transport formulations such as the one of \cite{courty14a,courty16a} can be seen as algorithmic solutions for controlling the trade-off between the terms of the bound. 

Assuming that $\epsilon_S (h)$ is properly minimized, only $\lambda$ and the Wasserstein distance between empirical measures defined on the source and target samples have an impact on the potential success of adaptation. Furthermore, the fact that the Wasserstein distance is defined in terms of the optimal coupling used to solve the DA problem and is not restricted to any particular hypothesis class directly influences $\lambda$ as discussed above. We now proceed to give similar bounds for the case where one has access to some labeled instances in the target domain. 

\subsection{A learning bound for the combined error}
In semi-supervised DA, when we have access to an additional small set of labeled instances in the target domain, the goal is often to find a trade-off between minimizing the source and the target errors depending on the number of instances available in each domain and their mutual correlation. Let us now assume that we possess $\beta n$ instances drawn independently from $\mu_T$ and $(1-\beta)n$ instances drawn independently from $\mu_S$ and labeled by $f_S$ and $f_T$, respectively. In this case, the empirical combined error \cite{bendavidth} is defined as a convex combination of errors on the source and target training data:
$$\hat{\epsilon}_{\alpha}(h) = \alpha \hat{\epsilon}_T(h) + (1-\alpha)\hat{\epsilon}_S(h),$$
where $\alpha \in [0,1]$.

The use of the combined error is motivated by the fact that if the number of instances in the target sample is small compared to the number of instances in the source domain (which is usually the case in DA), minimizing only the target error may not be appropriate. Instead, one may want to find a suitable value of $\alpha$ that ensures the minimum of $\hat{\epsilon}_{\alpha}(h)$ w.r.t. a given hypothesis $h$. We now prove a theorem for the combined error similar to the one presented in \cite{bendavidth}. 
\begin{theorem}
Under the assumptions of Theorem \ref{trm:our2} and Lemma \ref{trm:mmd_w}, let $D$ be a labeled sample of size $n$ with  $\beta n$ points drawn from $\mu_T$ and $(1-\beta) n$ from $\mu_S$ with $\beta \in (0,1)$,  and labeled according to $f_S$ and $f_T$.
If $\hat{h}$ is the empirical minimizer of $\hat{\epsilon}_\alpha(h)$ and $h^*_T = \underset {h}{\min} \ \epsilon_T(h)$ then for any $\delta \in (0,1)$ with probability at least $1-\delta$ (over the choice of samples),
$$\epsilon_T(\hat{h}) \leq \epsilon_T(h_T^*) + c_1 + 2(1-\alpha)(W_1(\hat{\mu}_S, \hat{\mu}_T) + \lambda + c_2),$$
where
\begin{align*}
\phantom{{}=1}c_1 = &2 \sqrt{\frac{2K\left(\frac{(1-\alpha)^2}{1-\beta}+\frac{\alpha^2}{\beta}\right)\log(2/\delta)}{n}} +4 \sqrt{K/n} \left( \frac{\alpha}{n\beta \sqrt{\beta} } + \frac{(1-\alpha)}{n(1-\beta)\sqrt{1-\beta} }\right),
\end{align*}
\begin{align*}
&c_2 = \sqrt{2\log\left(\frac{1}{\delta}\right)/\varsigma'}\left(\sqrt{\frac{1}{N_S}}+\sqrt{\frac{1}{N_T}}\right).
\end{align*}
\label{trm:our4}
\end{theorem}
\begin{proof}
\begin{align*}
\epsilon_T(\hat{h}) & \leq \epsilon_{\alpha}(\hat{h}) + (1-\alpha) (W_1(\mu_S,\mu_T)+\lambda)\\
& \leq \hat{\epsilon}_{\alpha}(\hat{h}) + \sqrt{\frac{2K\left(\frac{(1-\alpha)^2}{1-\beta}+\frac{\alpha^2}{\beta}\right)\log(2/\delta)}{n}}+ (1-\alpha) (W_1(\mu_S,\mu_T)+\lambda)
 \\
&\phantom{11}+2 \sqrt{K/n} \left( \frac{\alpha}{n\beta \sqrt{\beta} }+ \frac{(1-\alpha)}{n(1-\beta)\sqrt{1-\beta} }\right)\\
& \leq \hat{\epsilon}_{\alpha}(h_T^*) + \sqrt{\frac{2K\left(\frac{(1-\alpha)^2}{1-\beta}+\frac{\alpha^2}{\beta}\right)\log(2/\delta)}{n}}+ (1-\alpha)(W_1(\mu_S,\mu_T)+\lambda) \\
&\phantom{11}+ 2 \sqrt{K/n} \left( \frac{\alpha}{n\beta \sqrt{\beta} } + \frac{(1-\alpha)}{n(1-\beta)\sqrt{1-\beta} }\right)\\
& \leq \epsilon_{\alpha}(h_T^*) + 2\sqrt{\frac{2K\left(\frac{(1-\alpha)^2}{1-\beta}+\frac{\alpha^2}{\beta}\right)\log(2/\delta)}{n}}+ (1-\alpha)(W_1(\mu_S,\mu_T)+\lambda)\\
&\phantom{11}+4 \sqrt{K/n} \left( \frac{\alpha}{n\beta \sqrt{\beta} } + \frac{(1-\alpha)}{n(1-\beta)\sqrt{1-\beta} }\right)\\
& \leq \epsilon_{T}(h_T^*) + 2\sqrt{\frac{2K\left(\frac{(1-\alpha)^2}{1-\beta}+\frac{\alpha^2}{\beta}\right)\log(2/\delta)}{n}} +2(1-\alpha) (W_1(\mu_S,\mu_T)+\lambda) \\
&\phantom{11}+ 4 \sqrt{K/n} \left( \frac{\alpha}{n\beta \sqrt{\beta} } + \frac{(1-\alpha)}{n(1-\beta)\sqrt{1-\beta} }\right)\\
&\leq \epsilon_T(h_T^*) + c_1 + 2(1-\alpha)(W_1(\hat{\mu}_S, \hat{\mu}_T) + \lambda + c_2).
\end{align*}
The proof follows the standard theory of uniform convergence for empirical risk minimizers where lines 1 and 5 are obtained by observing that $ \vert \epsilon_{\alpha}(h) - \epsilon_T(h)\vert = \vert \alpha\epsilon_{T}(h) + (1-\alpha)\epsilon_S(h)-\epsilon_T(h) \vert = \vert (1-\alpha)(\epsilon_{S}(h)- \epsilon_T(h)) \vert \leq (1-\alpha)(W_1(\mu_T,\mu_{S})+\lambda)
$ where the last inequality comes from line 4 of the proof of Theorem~\ref{trm:our2}, line 3 follows from the definition of $\hat{h}$ and $h_T^*$ and line 6 is a consequence of 
Theorem \ref{trm:bolley}.
Finally, lines 2 and 4 are obtained based on the concentration inequality obtained for $\epsilon_{\alpha}(h)$. Due to the lack of space, we put this result in the Supplementary material. $\Box$
\end{proof}
This theorem shows that the best hypothesis that takes into account both source and target labeled data (\ie , $0 \leq \alpha < 1 $) performs at least as good as the best hypothesis learned on target data instances alone ($\alpha = 1$). This result agrees well with the intuition that semi-supervised DA approaches should be at least as good as unsupervised ones.

\section{Multi-source domain adaptation}
We now consider the case where not one but many source domains are available during the adaptation. More formally, we define $N$ different source domains (where $T$ can either be or not a part of this set).  
For each source $j$, we have a labelled sample $S_j$ of size $n_j = \beta_j n$ $\left( \sum_{j=1}^N \beta_j = 1, \sum_{j=1}^N n_j = n\right)$ drawn from the associated unknown distribution $\mu_{S_j}$ and labelled by $f_j$. We now consider the empirical weighted multi-source error of a hypothesis $h$ defined for some vector $\bm{\alpha} = \{\alpha_1, \dots, \alpha_N \}$ as follows:
$$\hat{\epsilon}_{\bm{\alpha}}(h) = \sum_{j=1}^N\alpha_j \hat{\epsilon}_{S_{j}}(h),$$
where $\sum_{j=1}^N\alpha_j = 1$ and each $\alpha_j$ represents the weight of the source domain $S_j$.

In what follows, we show that generalization bounds obtained for the weighted error give some interesting insights into the application of the Wasserstein distance to multi-source DA problems.
\begin{theorem}
With the assumptions from Theorem \ref{trm:our2} and Lemma \ref{trm:mmd_w}, let $S$ be a sample of size $n$, where for each $j \in \{1,\dots,N\}$, $\beta_j n$ points are drawn from $\mu_{S_j}$ and labelled according to $f_{j}$. If $\hat{h}_{\bm{\alpha}}$ is the empirical minimizer of $\hat{\epsilon}_{\bm{\alpha}}(h)$ and $h^*_T = \underset {h}{\min} \ \epsilon_T(h)$ then for any fixed $\bm{\alpha}$ and $\delta \in (0,1)$ with probability at least $1-\delta$ (over the choice of samples),
$$\epsilon_T(\hat{h}_{\bm{\alpha}}) \leq \epsilon_T(h_T^*) + c_1 + 2\sum_{j=1}^N \alpha_j \left( W_1(\hat{\mu}_j, \hat{\mu}_T)+\lambda_j+c_2\right),$$
where
$$c_1 = 2 \sqrt{\frac{2K\sum_{j=1}^N \frac{\alpha_j^2}{\beta_j} \log(2/\delta)}{n}}+2\sqrt{\sum_{j=1}^N\frac{K\alpha_j}{\beta_jn}},$$ 
$$c_2 = \sqrt{2\log\left(\frac{1}{\delta}\right)/\varsigma'}\left(\sqrt{\frac{1}{N_{S_j}}}+\sqrt{\frac{1}{N_T}}\right),$$
where $\lambda_j = \underset{h}{\min} \ (\epsilon_{S_j}(h)+\epsilon_T(h))$ represents the joint error for each source domain $j$.
\label{trm:our5}
\end{theorem}
\begin{proof}
The proof of this Theorem is very similar to the proof of Theorem \ref{trm:our5}. The final result is obtained by applying the concentration inequality for $\epsilon_{\bm{\alpha}}(h)$ (instead of those used for $\epsilon_{\alpha}(\hat{h})$ in the proof of Theorem \ref{trm:our5}) and by using the following inequality that can be obtained easily by following the principle of the proof of \cite[Theorem 4]{bendavidth}:
$$\vert \epsilon_{\bm{\alpha}}(h) - \epsilon_T(h) \vert \leq \sum_{j=1}^N\alpha_j \left( W_1(\mu_j, \mu_T) + \lambda_j \right),$$
where $\lambda_j = \underset{h}{\min} \ (\epsilon_{S_j}(h)+\epsilon_T(h))$.
For the sake of completness, we present the concentration inequality for $\epsilon_{\bm{\alpha}}(h)$ in the Supplementary material.$\Box$
\end{proof}
While the results for multi-source DA may look like a trivial extension of the theoretical guarantees for the case of two domains, they can provide a very fruitful implication on their own. As in the previous case, we consider that the potential term that should be minimized in this bound by a given multi-source DA algorithm is the term $\sum_{j=1}^N \alpha_j W_1(\hat{\mu}_j, \hat{\mu}_T)$.

Assume that $\hat{\mu}$ is an arbitrary unknown empirical probability measure on $\mathbb{R}^d$. Using the triangle inequality and bearing in mind that $\alpha_j\leq 1$ for all $j$, we can bound this term as follows:
$$\sum_{j=1}^N \alpha_j W_1(\hat{\mu}_j, \hat{\mu}_T) \leq (\sum_{j=1}^N \alpha_j W_1(\hat{\mu}_j, \hat{\mu})) + NW_1(\hat{\mu},\hat{\mu}_T).$$
Now, let us consider the following optimization problem 
\begin{align}
  \inf_{\hat{\mu} \in \mathcal{P}(\Omega)} \frac{1}{N} \sum_{j=1}^N \alpha_j W_1(\hat{\mu}_j, \hat{\mu})+ W_1(\hat{\mu},\hat{\mu}_T). 
  \label{eq1}
\end{align}
In this formulation, the first term  $\frac{1}{N} \sum_{j=1}^N \alpha_j W_1(\hat{\mu}_j, \hat{\mu})$ corresponds exactly to the problem known in the literature as the Wasserstein barycenters problem \cite{journals/siamma/AguehC11} that can be defined for $W_1$ as follows. 
\begin{definition} 
For $N$ probability measures $\mu_1, \mu_2, \dots, \mu_N \in \mathcal{P}(\Omega)$, an empirical Wasserstein barycenter is a minimizer $\mu^*_N \in \mathcal{P}(\Omega)$ of $J_N(\mu) = \min_{\mu} \frac{1}{N}\sum_{i=1}^N a_iW_1(\mu, \mu_i)$, where for all $i$, $a_{i}>0$ and $\sum_{i=1}^N a_i = 1$. 
\end{definition}
The second term $W_1(\hat{\mu},\hat{\mu}_T)$ of Equation \ref{eq1} finds the probability coupling that transports the barycenter to the target distribution. Altogether, this bound suggests that in order to adapt in the multi-source learning scenario, one can proceed by finding a barycenter of the source probability distributions and transport it to the target probability distribution. 

On the other hand, the optimization problem related to the Wasserstein barycenters is closely related to the Multimarginal optimal transportation problem \cite{pass_2011} where the goal is to find a probabilistic coupling that aligns $N$ distinct probability measures. Indeed, as shown in \cite{journals/siamma/AguehC11}, for a quadratic Euclidean cost function the solution $\mu^*_N$ of the barycenter problem in the Wasserstein space is given by the following equation:
$$\mu^*_N = \sum_{k \in \{k_1, \dots, k_N \}} \gamma_k \delta_{A_{k}(x)},$$
where $A_{k}(x) = \sum_{j=1}^N \gamma_j x_{k_j}$ and $\boldsymbol{\gamma} \in \mathbb{R}^{\prod_{j=1}^N n_j}$ is an optimal coupling solving for all $k \in \{1, \dots, N \}$ the multimarginal optimal transportation problem with the following cost:
$$c_k = \sum \frac{a_j}{2} \Vert x_{k_j} - A_{k}(x) \Vert^2.$$
We note that this reformulation is particularly useful when the source distributions are assumed to be Gaussians. In this case, there exists a closed form solution for the multimarginal optimal transportation problem \cite{knott_smith_94} and thus for Wasserstein barycenters problem too. Finally, it is also worth noticing that the optimization problem Equation \ref{eq1} has already been introduced to solve the multiview learning task\cite{abraham_2015}. In their formulation, the second term is referred to as an a priori knowledge about the barycenter which, in our case, is explicitly given by the target probability measure simultaneously.

\section{Comparison to other existing bounds}
As mentioned in the introduction, there are numerous papers that proposed DA generalization bounds. The main difference between them lies in the distance used to measure the divergence between source and target probability distributions. The seminal work of \cite{Ben-david07analysisof} considered a modification of the total variation distance called H-divergence given by the following equation:
$$d_H(p,q) = 2\sup_{h \in H} \vert p(h(x)=1) - q(h(x)=1)\vert.$$ On the other hand, \cite{DBLP:conf/colt/MansourMR09} and \cite{DBLP:journals/tcs/CortesM14} proposed to replace it with the discrepancy distance: 
$$\text{disc}(p,q) = \max_{h,h' \in H} \vert \epsilon_p(h,h') - \epsilon_q(h,h')\vert.$$
The latter one was shown to be tighter in some plausible scenarios. A more recent work on generalization bounds using integral probability metric 
$$\text{D}_{\mathcal{F}} (p,q) = \sup_{f \in \mathcal{F}} \vert \int fdp - \int fdq \vert$$
and R\'enyi divergence $$D_\alpha(p\Vert q) = \frac{1}{\alpha - 1} \log \left(\sum_{i=1}^n \frac{p_i^\alpha}{q_i^{\alpha - 1}}\right)$$ were presented in \cite{DBLP:conf/nips/ZhangZY12} and \cite{DBLP:conf/uai/MansourMR09}, respectively. \cite{DBLP:conf/nips/ZhangZY12} provides a comparative analysis of discrepancy and integral metric based bounds and shows that the former are less tight. \cite{DBLP:conf/uai/MansourMR09} derives the domain adaptation bounds in multisource scenario by assuming that the good hypothesis can be learned as a weighted convex combination of hypothesis from all the sources available. Considering a reasonable amount of previous work on the subject, a natural question about the tightness of the DA bounds based on the Wasserstein metric introduced above arises in spite of the Theorem 3. 

The answer to this question is partially given by the Csisz\`ar-Kullback-Pinsker inequlity \cite{opac-b1082909} defined for any two probability measures $p, q \in \mathcal{P}(\Omega)$ as follows:
$$W_1(p,q) \leq \text{diam}(\Omega)\Vert p-q \Vert_{\text{TV}} \leq \sqrt{2\text{diam}(\Omega)\text{KL}(p\Vert q)},$$
where $\text{diam}(\Omega) = \sup_{x,y \in \Omega} \{ d(x,y)\}$ and $\text{KL}(p\Vert q)$ is the Kullback-Leibler divergence.

A first consequence of this inequality shows that the Wasserstein distance not only appears naturally and offers algorithmic advantages in DA but also gives tighter bounds than total variation distance (L1) used in \cite[Theorem 1]{bendavidth}. On the other hand, it is also tighter than bounds presented in \cite{DBLP:conf/uai/MansourMR09} as the Wasserstein metric can be bounded by the Kullback-Leibler divergence which is a special case of R\'enyi divergence when $\alpha \rightarrow 1$ as shown in \cite{RePEc:eee:stapro:v:94:y:2014:i:c:p:77-85}. Regarding the discrepancy distance and omitting the hypothesis class restriction, one has $d_{min}\text{disc}(p,q) \leq W_1(p,q)$, where $d_{min} = \min_{x \neq y \in \Omega} \{ d(x,y)\}$. This inequality, however, is not very informative as minimum distance between two distinct points can be dramatically small thus making it impossible to compare the considered distances directly.

Regarding computational guarantees, we note that the H-divergence used in \cite{Ben-david07analysisof} is defined as the error of the best hypothesis distinguishing between the source and target domain samples pseudo-labeled with 0's and 1's and thus presents an intractable problem in practice. For the discrepancy distance, authors provided a linear time algorithm for its calculation in 1D case and showed that in other cases it scales as $O(N_S^2d^{2.5} + N_Td^2)$ when the squared loss is used \cite{DBLP:conf/colt/MansourMR09}. In its turn, the Wasserstein distance with entropic regularization can be calculated based on the linear time Sinkhorn-Knopp algorithm regardless the choice of the cost function $c$ that presents a clear advantage over the other distances considered above.

Finally, none of the distances previously introduced in the generalization bounds for DA take into account the geometry of the space meaning that the Wasserstein distance is a powerful and precise tool to measure the divergence between domains.

\section{Conclusion}
In this paper, we studied the problem of DA in the optimal transportation context. Motivated by the existing algorithmic advances in domain adaptation, we presented the generalization bounds for both single and multi-source learning scenarios where the distance between source and target probability distributions is measured by the Wasserstein metric. Apart from the distance term that taken alone justifies the use of optimal transport in domain adaptation, the obtained bounds also included the capability term depicting the existence of a good hypothesis for both source and target domains. A direct consequence of its appearance in the bounds is the need to regularize optimal transportation plan in a way that allows to ensure efficient learning in the source domain once the interpolation was done. This regularization, achieved in \cite{courty14a,courty16a} by the means of the class-based regularization, thus can be also viewed as an implication of the obtained results. Furthermore, it explains the superior performance of both class-based and Laplacian regularized optimal transport in domain adaptation compared to it simple entropy regularized form. On the other hand, we also showed that the use of the Wasserstein distance leads to tighter bounds compared to the bounds based on the total variation distance and R\'enyi divergence and is more computationally attractive than some other existing results. From the analysis of the bounds obtained for the multi-source DA, we derived a new algorithmic idea that suggests the minimization of two terms: first term corresponds to the Wasserstein barycenter problem calculated on the empirical source measures while the second one solves the optimal transport problem between this barycenter and the empirical target measure. 

Future perspectives of this work are many and concern both the derivation of new algorithms for domain adaptation and the demonstration of new theoretical results. First of all, we would like to study the extent to which the cost function used in the derivation of the bounds can be used on actual real-world DA problems. This distance, defined as a norm of difference between two feature maps, can offer a flexibility in the calculation of the optimal transport metric due to its kernel representation. Secondly, we aim to produce new concentration inequalities for the $\lambda$ term that will allow to bound the true best joint hypothesis by its empirical counter-part. These concentration inequalities will allow to access the adaptability of two domains from the given labelled samples while the speed of convergence may show how many data instances from the source domains is needed to obtain a reliable estimate of $\lambda$. Finally, the introduction of the Wasserstein distance to the bounds means that new DA algorithms can be designed based on the other optimal coupling techniques. These include, for instance, Knothe-Rosenblatt coupling and Moser coupling. 

\paragraph{\bf Acknowledgments.} This work was supported in part by the French ANR project LIVES ANR-15-CE23-0026-03.

\allowdisplaybreaks
\mainmatter  

\title{Supplementary material: Theoretical Analysis of Domain Adaptation with Optimal Transport}


%
%

\author{Ievgen Redko\inst{1} \and Amaury Habrard\inst{2} \and Marc Sebban \inst{2}}

%
\authorrunning{} 
%
%

\institute{Univ.Lyon, INSA‐Lyon, Universit\'e Claude Bernard Lyon 1, UJM-Saint Etienne\\
CNRS, Inserm, CREATIS UMR 5220, U1206\\
F-69266, LYON, France\\
\email{ievgen.redko@creatis.insa-lyon.fr}\\ 
\and
Univ. Lyon, UJM-Saint-Etienne\\
   CNRS, Lab. Hubert Curien UMR 5516\\
   F-42023, SAINT-ETIENNE, France\\
   \email{\{amaury.habrard,marc.sebban\}@univ.st-etienne.fr}}


\maketitle
\phantom{1}
We start by presenting the original version of Theorem 1. Then, we proceed by introducing the concentration results for combined errors for single and multi-source settings.

\section{Original formulation of Theorem 1}
\setcounter{theorem}{4}
\begin{theorem}[\cite{Bolley:2007:QCI}, Theorem 1.1]
Let $p \in [1;2]$ and $\mu$ be a probability measure in $\mathbb{R}^d$ satisfying $T_p(\lambda)$ inequality. Then for any $d'>d$ and $\lambda' < \lambda$ there exists some constant $N_0$ depending on $d',\lambda'$and some square exponential moment of $\mu$ such that for any $\varepsilon > 0$ and $N \geq N_0 \max(\varepsilon^{-(d'+2)},1)$
$$\mathbb{P} \left[ W_1(\mu, \hat{\mu}) > \varepsilon \right] \leq \exp\left(-\gamma_p\frac{\lambda'}{2} N  \varepsilon^2\right),$$
\label{trm_concen}
where 
$$    \gamma_p=\left\{
                \begin{array}{ll}
                  1, \ 1 \leq p < 2,\\
                  3-2\sqrt{2}, \ p = 2.\\
               \end{array}
              \right.
$$
\end{theorem}
In this Theorem the condition $T_p(\lambda)$ means that given  $p\geq 1$, $\lambda > 0$ and a probability measure $\mu$ on $X$, the Talagrand inequality 
$$W_p(\mu, \nu) \leq \sqrt{\lambda \text{KL}(\mu,\nu)}$$
holds for any probability measure $\nu$.
\section{Concentration inequality used in the Proof of Theorem 3}
\setcounter{page}{1}
\setcounter{lemma}{1}
\begin{lemma}
Under the assumptions of Lemma 1, let $D$ be a sample of size $n$ with  $\beta n$ points  drawn from $\mu_T$ and $(1-\beta) n$ from $\mu_S$, $\beta\in[0,1]$ and labeled according to $f_S$ and $f_T$. Then with probability at least $1-\delta$ for all $h$ the following holds:
\begin{align*}
P&\left\lbrace  \vert \hat{\epsilon}_{\alpha}(h) - \epsilon_{\alpha}(h ) \vert > 2 \sqrt{K/n} \left( \frac{\alpha}{n\beta \sqrt{\beta} } + \frac{(1-\alpha)}{n(1-\beta)\sqrt{1-\beta} }\right) +\epsilon \right\rbrace \\  
& \phantom{{}=aaaaaaaaaaaaaaaaaaaaaaaaa} \leq \exp \left\lbrace {\frac{-\epsilon^2n}{2K\left(\frac{(1-\alpha)^2}{1-\beta}+\frac{\alpha^2}{\beta}\right)}} \right\rbrace.
\end{align*}
\label{trm:our4}
\end{lemma}
\begin{proof}
First, we use McDiarmid's inequality in order to obtain the right side of the inequality by defining the maximum changes of magnitude when one of the sample vectors has been changed. 

For the sake of completeness, we give its definition here.
\setcounter{definition}{1}
\begin{definition}\label{def:MD}
Suppose $X_1, X_2, \dots, X_n$ are independent random variables taking values in a set $A$ and assume that $f:A_n \rightarrow \mathbb{R}$ satisfies for $\bm{x} = \left[ x_1, x_2, \dots, x_n \right]$ and $\bm{x}_i = \left[x_1, x_2, \dots, x_{i-1}, \hat{x}_i, x_{i+1}, \dots, x_n \right]$\phantom{{}=1} 
\begin{align*}
\sup_{x_1, x_2, \dots, x_n, \hat{x}_i} \vert f(\bm{x}) &- f(\bm{x}_i) \vert \leq c_i, \ \text{for} \ 1 \leq i \leq n, 
\end{align*}

then the following inequality holds for any $\varepsilon > 0$
\begin{align*}
    P&\left\lbrace \vert f(x_1, x_2, \dots, x_n) - \mathbb{E}\left[f(x_1, x_2, \dots, x_n)\right] \vert > \varepsilon \right\rbrace \leq \exp \left\lbrace \frac{-2\epsilon^2}{\sum_{i=1}^n c_i^2} \right\rbrace.
\end{align*}
\end{definition}

We first rewrite the difference between the empirical and true combined error in the following way
\begin{align*}
& \vert \hat{\epsilon}_{\alpha}(h) - \epsilon_{\alpha}(h) \vert = \vert \alpha (\epsilon_T(h) - \hat{\epsilon}_{T}(h)) - (\alpha-1)(\epsilon_{S}(h) -\hat{\epsilon}_{S}(h))\vert \\
& = \vert \alpha \mathbb{E}_{x \sim \mu_T} \left[l(h(x),f(x))\right] - (\alpha-1) \mathbb{E}_{y \sim \mu_S} \left[l(h(y),f(y))\right]\\
& - \frac{\alpha}{n\beta } \sum_{i=1}^{\beta n} l(h(x_i),f_T(x_i)) + \frac{(\alpha-1)}{n(1-\beta) } \sum_{i=1}^{n(1-\beta) } l(h(y_i),f_S(y_i))\vert\\
& \leq \sup_{l \in \mathcal{H}} \ \vert \alpha \mathbb{E}_{x \sim \mu_T} \left[l(h(x),f(x))\right] - (\alpha-1) \mathbb{E}_{y \sim \mu_S} \left[l(h(y),f(y))\right] \\
&- \frac{\alpha}{n\beta } \sum_{i=1}^{n\beta }  l(h(x_i),f_T(x_i)) + \frac{(\alpha-1)}{n(1-\beta) } \sum_{i=1}^{n(1-\beta) } l(h(y_i),f_S(y_i))\vert.
\end{align*}

Changing either $x_i$ or $y_i$ in this expression changes its value by at most $\frac{2\alpha \sqrt{K}}{\beta n}$ and $\frac{2(1-\alpha)\sqrt{K}}{(1-\beta) n}$, respectively. This gives us the denominator of the exponential in Definition~\ref{def:MD}
\begin{align*}
\beta n \left( \frac{2\alpha \sqrt{K}}{\beta n} \right)^2 + (1-&\beta) n \left( \frac{2(1-\alpha)\sqrt{K}}{(1-\beta) n} \right)^2 = \frac{4K}{n}\left( \frac{\alpha^2}{\beta} + \frac{(1-\alpha)^2}{(1-\beta)}\right).
\end{align*}
Then, we bound the expectation of the difference between the true and empirical combined errors by the sum of Rademacher averages over the samples. Denoting by $X'$ an i.i.d sample of size $\beta m$ drawn independently of $X$ (and likewise for $Y'$), and using the symmetrization technique we have

\begin{align*}
& \mathbb{E}_{X,Y} \sup_{h \in \mathcal{H}} \ \vert \alpha \mathbb{E}_{x \sim \mu_T} \left[l(h(x),f(x))\right] - (\alpha-1) \mathbb{E}_{y \sim \mu_S} \left[l(h(y),f(y))\right]\\ 
& \phantom{{}=1} - \frac{\alpha}{n\beta} \sum_{i=1}^{ n\beta} l(h(x_i),f_S(x_i)) + \frac{(\alpha-1)}{n(1-\beta) } \sum_{i=1}^{n(1-\beta) } l(h(y_i),f_T(y_i))\vert \\
& \leq \mathbb{E}_{X,Y} \sup_{h \in \mathcal{H}} \ \vert \mathbb{E}_{X'} \left( \frac{\alpha}{n\beta } \sum_{i=1}^{n\beta } l(h(x_i'),f_S(x_i')) \right)\\
&\phantom{{}=1} - (\alpha-1) \mathbb{E}_{Y'} \left( \frac{(\alpha-1)}{n(1-\beta) } \sum_{i=1}^{n\beta } l(h(y_i'),f_T(y_i')) \right) \\
& \phantom{{}=1} - \frac{\alpha}{n\beta } \sum_{i=1}^{\beta n} l(h(x_i),f_S(x_i)) + \frac{(\alpha-1)}{n(1-\beta) } \sum_{i=1}^{(1-\beta) n} l(h(y_i),f_T(y_i)) \vert \\
& \leq \mathbb{E}_{X,X',Y,Y'} \sup_{h \in \mathcal{H}} \ \vert \frac{\alpha}{n\beta } \sum_{i=1}^{\beta n} \sigma_i(l(h(x_i'),f_S(x_i')) - l(h(x_i),f_S(x_i))) \\
& \phantom{{}=1} + \frac{1-\alpha}{n(1-\beta)} \sum_{i=1}^{\beta n} \sigma_i(l(h(y_i'),f_T(y_i')) - l(h(y_i),f_T(y_i)))\vert\\
& \leq 2 \sqrt{K/n} \left( \frac{\alpha}{n\beta \sqrt{\beta} } + \frac{(1-\alpha)}{n(1-\beta)\sqrt{1-\beta} }\right).
\end{align*}

Finally, the Rademacher averages, in their turn, are bounded using a theorem from \cite{Bartlett:2003:RGC:944919.944944}. Using this inequality in Definition~\ref{def:MD} gives us the desired result:

\begin{align*}
P&\left\lbrace  \vert \hat{\epsilon}_{\alpha}(h) - \epsilon_{\alpha}(h ) \vert > 2 \sqrt{K/n} \left( \frac{\alpha}{n\beta \sqrt{\beta} } + \frac{(1-\alpha)}{n(1-\beta)\sqrt{1-\beta} }\right) +\epsilon \right\rbrace \\  
& \phantom{{}=aaaaaaaaaaaaaaaaaaaaaaaaa} \leq \exp \left\lbrace {\frac{-\epsilon^2n}{2K\left(\frac{(1-\alpha)^2}{1-\beta}+\frac{\alpha^2}{\beta}\right)}} \right\rbrace.
\end{align*}
$\Box$
\end{proof}
\section{Concentration inequality used in the Proof of Theorem 4}
\begin{lemma}
Under the assumptions of Lemma 1, let $D$ be a sample of size $n$, where for each $j \in \{1,\dots,N\}$, $\beta_j n$ points are drawn from $\mu_{S_j}$ and labeled according to $f_{j}$. Then for any fixed $\bm{\alpha}$, with probability at least $1-\delta$ for all $h$ the following holds:
\begin{align*}
P&\left\lbrace  \vert \hat{\epsilon}_{\alpha}(h) - \epsilon_{\alpha}(h ) \vert > 2 \sqrt{K/n} \sum_{j=1}^N\frac{\alpha_j}{\beta_j n \sqrt{\beta_j} } +\epsilon \right\rbrace \\
& \phantom{{}=aaaaa1aaaaaaaaaaa} \leq \exp \left\lbrace {\frac{-\epsilon^2n}{2K \sum_{j=1}^N \frac{\alpha_j^2}{\beta_j}}} \right\rbrace.
\end{align*}
\label{trm:our6}
\end{lemma}
\begin{proof}
The proof of this Lemma is very similar to the proof of Lemma \ref{trm:our4}. Two main differences lie in how the denominator of the exponent is defined. For a fixed vector $\bm{\alpha}$, it is equal to $\frac{4K}{n}\left( \sum_{j=1}^N \frac{\alpha_j^2}{\beta_j} \right)$. Similarly, we obtain the bound for the Rademacher complexities that equals to $2 \sqrt{K/n} \sum_{j=1}^N\frac{\alpha_j}{\beta_j n \sqrt{\beta_j} }$.
$\Box$
\end{proof}

\end{document}